\documentclass[10pt,journal,compsoc]{IEEEtran}

\usepackage{subfigure}
\usepackage{graphicx}
\usepackage{multirow}
\usepackage{algorithm}
\usepackage{algorithmic}
\usepackage{url}
\usepackage{balance}
\usepackage{booktabs}
\usepackage{proof}
\usepackage{amsmath,amssymb,amsfonts}
\newtheorem{theorem}{Theorem}
\newtheorem{proof}{Proof}

\ifCLASSOPTIONcompsoc
  \usepackage[nocompress]{cite}
\else
  \usepackage{cite}
\fi
\ifCLASSINFOpdf
\else
\fi
\begin{document}
\title{RMT-Net: Reject-aware Multi-Task Network for Modeling Missing-not-at-random Data in Financial Credit Scoring}

\author{Qiang~Liu,~\IEEEmembership{Member,~IEEE},
        Yingtao~Luo,
        Shu~Wu,~\IEEEmembership{Senior Member,~IEEE},
        Zhen~Zhang,
        Xiangnan~Yue,
        Hong~Jin,
        and~Liang~Wang,~\IEEEmembership{Fellow,~IEEE}
\IEEEcompsocitemizethanks{
\IEEEcompsocthanksitem Qiang Liu, Shu Wu and Liang Wang are with the Center for Research on Intelligent Perception and Computing (CRIPAC), National Laboratory of Pattern Recognition (NLPR), Institute of Automation, Chinese Academy of Sciences (CASIA), Beijing, China.\protect\\
E-mail: \{{qiang.liu, shu.wu, wangliang}\}@nlpr.ia.ac.cn
\IEEEcompsocthanksitem Yingtao Luo is with H. John Heinz III School of Information Systems and Management, Carnegie Mellon University, Pittsburgh, USA.\protect\\
E-mail: yingtaoluo@cmu.edu
\IEEEcompsocthanksitem Zhen Zhang, Xiangnan Yue and Hong Jin are with Ant Group, Hangzhou, China.\protect\\
E-mail: yilue.zz@antfin.com, yuexiangnan.yxn@antgroup.com, jinhong.jh@alibaba-inc.com
}
\thanks{
% Manuscript received October 15, 2021; revised February 15, 2022.\\
(Corresponding authors: Shu Wu and Zhen Zhang)}}


\markboth{Journal of \LaTeX\ Class Files,~Vol.~14, No.~8, August~2015}%
{Shell \MakeLowercase{\textit{et al.}}: Bare Demo of IEEEtran.cls for Computer Society Journals}



\IEEEtitleabstractindextext{%
\begin{abstract}
In financial credit scoring, loan applications may be approved or rejected.
We can only observe default/non-default labels for approved samples but have no observations for rejected samples, which leads to missing-not-at-random selection bias.
Machine learning models trained on such biased data are inevitably unreliable.
In this work, we find that the default/non-default classification task and the rejection/approval classification task are highly correlated, according to both real-world data study and theoretical analysis.
Consequently, the learning of default/non-default can benefit from rejection/approval.
Accordingly, we for the first time propose to model the biased credit scoring data with Multi-Task Learning (MTL).
Specifically, we propose a novel Reject-aware Multi-Task Network (RMT-Net), which learns the task weights that control the information sharing from the rejection/approval task to the default/non-default task by a gating network based on rejection probabilities.
RMT-Net leverages the relation between the two tasks that the larger the rejection probability, the more the default/non-default task needs to learn from the rejection/approval task.
Furthermore, we extend RMT-Net to RMT-Net++ for modeling scenarios with multiple rejection/approval strategies.
Extensive experiments are conducted on several datasets, and strongly verifies the effectiveness of RMT-Net on both approved and rejected samples.
In addition, RMT-Net++ further improves RMT-Net's performances.
\end{abstract}

}

\maketitle

\IEEEdisplaynontitleabstractindextext

% \fi
\IEEEpeerreviewmaketitle

\IEEEraisesectionheading{\section{Introduction}\label{sec:introduction}}

\IEEEPARstart{C}{redit} scoring aims to use machine learning methods to measure customers' default probabilities of credit loans \cite{west2000neural}\cite{hu2020loan}\cite{liu2020alike}\cite{liu2021dnn2lr}\cite{babaev2019rnn}.
Based on the evaluated credits, financial institutions such as banks and online lending companies can decide whether to approve or reject credit loan applications.

When a customer applies for credit loan, his or her application may be approved or rejected.
If the application is approved, it will become an \textit{approved sample}, and the customer will get the loan.
After a period, if the customer repays the credit loan timely, it will be a \textit{non-default sample}; if the customer fails to timely repay, it will be a \textit{default sample}.
In contrast, if the application is not approved, it will become a \textit{rejected sample}, and the customer will not get credit loan.
Since a rejected sample gets no loans, we have no way to observe whether it will be default or non-default.
Above process is illustrated in Fig. \ref{pic:task}.
Credit scoring models are usually constructed based on approved samples, as we have no ground-truth default/non-default labels for rejected samples \cite{li2017reject}\cite{mancisidor2020deep}\cite{ehrhardt2021reject}\cite{goel2021importance}.
The rejection/approval strategies are usually machine learning models or expert rules based on the features of customers, thus approved and rejected samples share different feature distributions.
This makes us face the \textit{missing-not-at-random selection bias} in data \cite{goel2021importance}\cite{marlin2009collaborative}\cite{schnabel2016recommendations}. 
However, when serving online, credit scoring models need to infer credits of loan applications in feature distributions of both approved and rejected samples.
Training models with such biased data has severe consequences that the model parameters are biased \cite{bucker2013reject}, i.e., the predicted relation between input features and default probability is incorrect.
Using such models on samples across various data distributions leads to significant economic losses \cite{mancisidor2020deep}\cite{chen2001economic}\cite{nguyen2016reject}.
Therefore, for reliable credit scoring, besides the modeling of approved samples, we also need to take rejected ones into consideration and infer their true credits \cite{hand1993can}.

In practice, machine learning models like Logistic Regression (LR), Support Vector Machines (SVM), Multi-Layer Perceptron (MLP) and XGBoost (XGB) are widely used for modeling credit scoring data. However, they are affected by the missing-not-at-random bias in data to produce reliable and accurate predictions.
To tackle this problem, some existing approaches address the selection bias and conduct reject inference from multiple perspectives.
Some approaches apply the self-training algorithm \cite{agrawala1970learning}, which iteratively adds rejected samples with higher default probability as default samples to retrain the model \cite{maldonado2010semi}.
This is a semi-supervised approach \cite{zhu2009introduction}.
Besides, Semi-Supervised SVM (S3VM) \cite{li2017reject} and Semi-Supervised Gaussian Mixture Models (SS-GMM) \cite{mancisidor2020deep} are also deployed in credit scoring systems.
In another perspective, some approaches attempt to re-weight the training approved samples to approximate unbiased data \cite{nguyen2016reject}\cite{hsia1978credit}\cite{banasik2007reject}\cite{banasik2003sample}.
These approaches are similar to counterfactual learning \cite{marlin2009collaborative}\cite{schnabel2016recommendations}\cite{hassanpour2019counterfactual}\cite{zou2020counterfactual}, which attempts to re-weight observed samples to remove bias in data.

\begin{figure}
\centering
\includegraphics[width=0.48\textwidth]{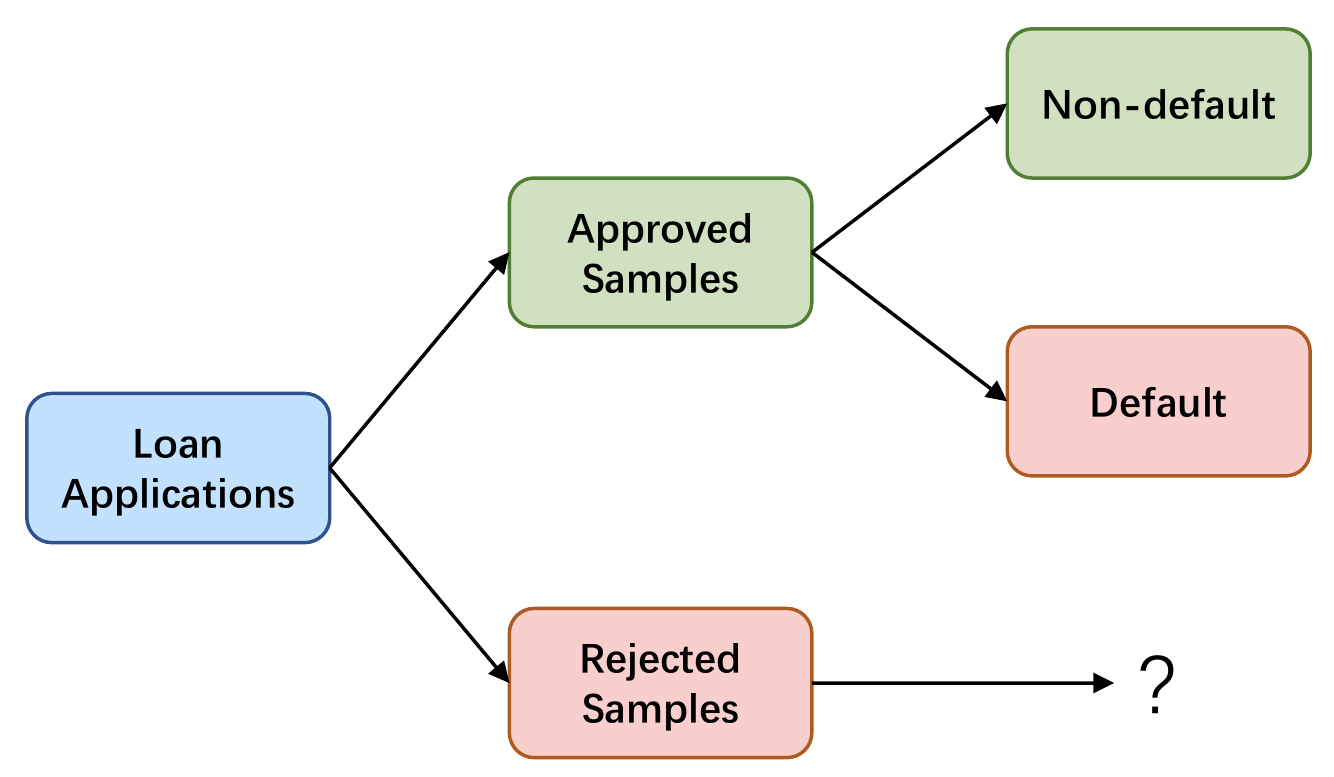}
\caption{Illustration of data bias in credit scoring.}
\label{pic:task}
\end{figure}

Though some of the above approaches have achieved relative improvements on some credit scoring datasets \cite{mancisidor2020deep}\cite{nguyen2016reject}, they cannot achieve optimal performances due to the lack of consideration of some key factors.
Specifically, we find that the \textbf{default/non-default} classification task and the \textbf{rejection/approval} classification task are highly correlated in real credit scoring applications, according to both real-world data study and theoretical analysis in Sec. \ref{sec:analysis}.
Intuitively speaking, with an effective credit approval system, rejected customers have higher default ratios, while approved customers have lower ones.
Consequently, \textbf{the learning of default/non-default can benefit from the learning of rejection/approval}.
Accordingly, it might be promising to incorporate \textit{Multi-Task Learning (MTL)} \cite{caruana1997multitask} for modeling biased credit scoring data.

Nowadays, state-of-the-art MTL approaches mainly focus on adaptively learning weights of different tasks in a mixture-of-experts structure \cite{ma2018modeling}\cite{tang2020progressive}\cite{xi2021modeling}\cite{zhao2019multiple}\cite{liu2019end}.
This makes task weights changing in different samples so that tasks can share useful but not conflict information adaptively.
Such MTL approaches achieve promising performances in various scenarios.
However, when we use state-of-the-art MTL approaches for modeling the default/non-default task and the rejection/approval task, we do not achieve satisfactory performances, and even achieve poor performances in default prediction on rejected samples.
This may be because we have no observed default/non-default labels for rejected samples during model training. \textbf{The task weights, which decide how much information is shared between the two tasks, are not well optimized in the feature distribution of rejected samples}.
Thus, exiting MTL approaches fail in modeling the biased credit scoring data, and we need a novel and specially-designed MTL approach.

Accordingly, we propose a \textbf{Reject-aware Multi-Task Network (RMT-Net)}.
RMT-Net learns the weights that control the information sharing from the rejection/approval task to the default/non-default task by a gating network based on rejection probabilities. 
With larger rejection probability, less reliable information can be learned in the default/non-default network and more information is shared from the rejection/approval network.
In this way, we can consider the correlation between rejected samples and default samples, as well as personalize the information sharing weights in the feature distribution of rejected samples.
Furthermore, we consider cases with \textit{multiple rejection/approval strategies}, and extend RMT-Net to \textbf{RMT-Net++}, which models several rejection/approval classification tasks in the MTL framework.

In all, we verify RMT-Net and RMT-Net++ on $10$ datasets under different settings, in which significant improvements are achieved for default prediction on both accepted and rejected samples.
Evaluated by the commonly-used Kolmogorov-Smirnov (KS) metric\footnote{\url{https://en.wikipedia.org/wiki/Kolmogorov–Smirnov_test}} in credit scoring, comparing with conventional classifiers, i.e. LR, DNN, and XGB, RMT-Net relatively improves the performances by $47.9\%$ on average.
Comparing with the most competitive reject inference approaches, RMT-Net relatively improves the performances by $11.9\%$ on average.
In addition, we show in an extra experiment with multiple rejection/approval strategies that RMT-Net++ can further relatively improve the performances of RMT-Net by $5.8\%$ on average.

The main contributions of this work are concluded:
\begin{itemize}
    \item We for the first time propose to model biased credit scoring data using an MTL approach, namely RMT-Net. Instead of directly using conventional MTL approaches, we present several modifications to improve the poor performances of existing MTL approaches on credit scoring.
    \item We further consider multiple rejection/approval strategies, and extend RMT-Net to RMT-Net++. In this way, our work suits different application scenarios in real applications.
    \item Extensive experiments are conducted on $10$ datasets under different settings. Significant improvements are achieved by our proposed RMT-Net approach on both accepted and rejected samples. In addition, we show that RMT-Net++ with multiple strategies can further improve the performances.
\end{itemize}

The rest of the paper is organized as follows. In Section 2, we review some related work on reject inference, counterfactual learning and multi-task learning.
Then we analyze the correlation between the default/non-default task and the rejection/approval task according to both real-world data study and theoretical analysis in Section 3.
Sections 4 and 5 detail our proposed RMT-Net and RMT-Net++ under single strategy and multiple strategies respectively.
In Section 6, we conduct empirical experiments to verify the effectiveness of RMT-Net and RMT-NET++.
Section 7 concludes our work.

\section{Related Work}

In this section, we review some works on reject inference, as well as two important related research aspects: counterfactual learning and multi-task learning.

\subsection{Reject Inference}

In the credit scoring task, we have only ground-truth default/non-default labels for approved samples but no ground-truth default/non-default labels for rejected samples.
This causes the missing-not-at-random bias in data \cite{goel2021importance}\cite{marlin2009collaborative}\cite{schnabel2016recommendations} for machine learning models. Some reject inference approaches are accordingly proposed \cite{ehrhardt2021reject}\cite{nguyen2016reject}\cite{hand1993can}.

Augmentation is a re-weighting approach \cite{hsia1978credit}\cite{banasik2007reject}\cite{banasik2003sample}, in which accepted samples are re-weighted to represent the entire distribution.
A common way to achieve this is re-weighting according to the rejection/approval probability.
Moreover, the augmentation approach has been extended in a fuzzy way \cite{nguyen2016reject}.
Parcelling is also a re-weighting approach, where the re-weighting is determined by the default probability by score-band that is adjusted by the credit modeler \cite{ehrhardt2021reject}\cite{banasik2003sample}.
To be noted, these re-weighting methods are similar to the researches on counterfactual learning \cite{marlin2009collaborative}\cite{schnabel2016recommendations}\cite{hassanpour2019counterfactual}\cite{zou2020counterfactual}.
Counterfactual learning aims to remove data bias, in which the re-weighting of training samples is widely adopted.

Meanwhile, semi-supervised approaches are also applied to deal with the reject inference task.
In \cite{maldonado2010semi}, the authors use a self-training algorithm to improve the performance of SVM on credit scoring.
Self-training, also known as self-labeling or decision-directed learning, is the most simple semi-supervised learning method \cite{agrawala1970learning}\cite{culp2008iterative}\cite{haffari2012analysis}.
This approach trains a model on approved samples, and labels rejected samples with largest default probabilities as default samples according to model predictions.
Then, the newly labeled samples are added to retrain the model, and this process continues iteratively.
Though the self-training algorithm is only used to promote SVM in \cite{maldonado2010semi}, it can also promote other classifiers, such as LR, MLP and XGB.
Besides, another semi-supervised version of SVM called S3VM \cite{li2017reject} is also applied in reject inference.
S3VM uses approved and rejected samples to fit an optimal hyperplane with maximum margin, but have problem in fitting large-scale data \cite{mancisidor2020deep}.
Meanwhile, earlier works have used some statistical machine learning methods, such as Expectation-Maximization (EM) algorithm \cite{anderson2013modified}, Gaussian Mixture Models (GMM) \cite{feelders2000credit} and survival analysis \cite{sohn2006reject}, for reject inference.
Based on GMM and inspired by semi-supervised generative models \cite{rezende2014stochastic}\cite{kingma2014semi}, SS-GMM \cite{mancisidor2020deep} is proposed for modeling biased credit scoring data. The counterfactual re-weighting and semi-supervised learning are the main methods for reject inference, but neither approach considers the correlation between the learning of rejection/approval and the learning of default/non-default.

\subsection{Counterfactual Learning}

Counterfactual learning \cite{zou2020counterfactual} is a key direction of the research on causal inference \cite{pearl2009causal}\cite{morgan2015counterfactuals}.
Counterfactual learning aims to simulate counterfactuals to alleviate the missing-not-at-random bias for a less biased model training \cite{little2019statistical}.
In the context of credit scoring, we know that rejected samples are unobserved, but counterfactual learning tries to answer ``what if they are observable?''

The traditional counterfactual learning approaches usually re-weight samples based on propensity scores \cite{hassanpour2019counterfactual}\cite{zou2020counterfactual}\cite{austin2011introduction}\cite{rosenbaum1983central}\cite{hassanpour2019learning}\cite{lopez2017estimation}.
Propensity scores indicate the probabilities of observation under different environments, e.g., approval and rejection in the credit scoring scheme.
These propensity score-based methods and re-weighting approaches in reject inference \cite{hsia1978credit}\cite{banasik2007reject}\cite{banasik2003sample} are similar, and both try to balance the data distributions of observed and unobserved samples.
Stable learning is another perspective of counterfactual learning, in which there is no implicit treatments and the distribution of unobserved samples is unknown \cite{kuang2018stable}.
Stable learning is usually done via decorrelation among features of samples, which tries to make the feature distribution closer to independently identically distribution \cite{kuang2020stable}\cite{shen2020stable}\cite{zhang2021deep}.
Meanwhile, Sample Reweighted Decorrelation Operator (SRDO) \cite{shen2020stable} generates some unobserved samples, and trains a binary classifier to get the probabilities of observation for re-weighting the observed samples.
This is somehow similar to the propensity score-based approaches.

Missing-not-at-random selection bias is also frequently discussed in recommender systems, where we can only observe feedback of displayed user-item pairs \cite{marlin2009collaborative}\cite{sato2020unbiased}.
In counterfactual recommendation, propensity score is also applied, and the Inverse Propensity Score (IPS) approach \cite{schnabel2016recommendations}\cite{swaminathan2015self} that re-weights observed samples with the inverse of displayed probabilities is proposed.
Based on IPS, Doubly Robust (DR) \cite{jiang2016doubly} and Joint Learning Doubly Robust (DRJL) \cite{wang2019doubly} are proposed to consider doubly robust estimator.
After that, some improvements have been presented, such as asymmetrically tri-training \cite{saito2020asymmetric}, considering information theory \cite{wang2020information}, and proposing better doubly robust estimators \cite{guo2021enhanced}.
Meanwhile, the Adversarial Counterfactual Learning (ACL) approach \cite{xu2020adversarial} incorporates adversarial learning for counterfactual recommendation.
Besides, some works on counterfactual recommendation rely on a small amount of random unbiased data \cite{bonner2018causal}\cite{yuan2019improving}\cite{chen2021autodebias}.
However, random data requires high costs, especially in financial applications.

\subsection{Multi-task Learning}

MTL learns multiple tasks simultaneously in one model, and has been proven to improve performances through information sharing between tasks \cite{caruana1997multitask}\cite{tang2020progressive}.
It has succeed in scenarios such as computer vision \cite{liu2019end}\cite{misra2016cross}\cite{ruder2017sluice}, recommender systems \cite{ma2018modeling}\cite{tang2020progressive}\cite{xi2021modeling}\cite{zhao2019multiple}\cite{liu2017multi}\cite{wen2020entire}, healthcare \cite{liu2019complication}, and other prediction problems \cite{zhao2017feature}\cite{xu2017online}.

The simplest MTL approach is hard parameter sharing, which shares hidden representations across different tasks, and only the last prediction layers are special for different tasks \cite{caruana1997multitask}.
However, hard parameter sharing suffers from conflicts among tasks, due to the simple sharing of representations.
To deal with this problem, some approaches propose to learn weights of linear combinations to fuse hidden representations in different tasks, such as Cross-Stitch Network \cite{misra2016cross} and Sluice Network \cite{ruder2017sluice}.
However, in different samples, the weights of different tasks stay the same, which limits the performances of MTL.
This inspires the research on applying gating structures in MTL \cite{ma2018modeling}\cite{tang2020progressive}\cite{xi2021modeling}\cite{jacobs1991adaptive}.
Mixture-Of-Experts (MOE) first proposes to share and combine several experts through a gating network \cite{jacobs1991adaptive}.
Based on MOE, to make the weights of different tasks varying across different samples and to improve the performances of MTL, Multi-gate MOE (MMOE) \cite{ma2018modeling} proposes to use different gates for different tasks.
Progressive Layered Extraction (PLE) further extends MMOE, and incorporates multi-level experts and gating networks \cite{tang2020progressive}.
Besides, attention networks are also utilized for assigning weights of tasks according to different feature representations \cite{zhao2019multiple}\cite{liu2019end}.

\begin{table*}[t]
  \centering
  \caption{The mean values of Fico score, debt-to-income ratio, loan amount and employment length of approved and rejected customers in the Lending Club dataset across different years.}
    \begin{tabular}{ccccccccc}
    \toprule
    \multirow{2}[2]{*}{Year} & \multicolumn{2}{c}{Fico Score} & \multicolumn{2}{c}{Debt-To-Income Ratio (\%)} & \multicolumn{2}{c}{Loan Amount} & \multicolumn{2}{c}{Employment Length (year)} \\
          & Approved & Rejected & Approved & Rejected & Approved & Rejected & Approved & Rejected \\
    \midrule
    2013  & 692   & 649   & 15.04  & 20.22  & 13921   & 13271   & 2.34  & 1.68 \\
    2014  & 687   & 638   & 14.76  & 19.88  & 12386   & 12040   & 2.14  & 1.58 \\
    2015  & 687   & 640   & 14.71  & 22.03  & 13542   & 13663   & 2.55  & 1.56 \\
    2016  & 692   & 636   & 14.43  & 24.14  & 13912   & 13698   & 2.70  & 1.37 \\
    2017  & 692   & 636   & 14.20  & 23.00  & 12644   & 12501   & 2.40  & 1.26 \\
    2018  & 702   & 632   & 14.56  & 20.37  & 13897   & 13385   & 3.28  & 1.23 \\
    \bottomrule
    \end{tabular}%
  \label{tab:data_stydy}%
\end{table*}%

\section{Analysis} \label{sec:analysis}
In this section, we plan to analyze the correlation between the \textbf{default/non-default} task and the \textbf{rejection/approval} task.
First, we analyze the correlation between the rejected samples and the approved samples in real-world datasets, in which we can conclude that users whose loan applications are rejected are more likely to default.
Then, we theoretically prove that the reject/approve task and the default/non-default task are correlated, so that we are motivated to model the reject/approve task and the default/non-default task by multi-task learning.

\subsection{Data Study}
With analyses of public real-world datasets, we plan to illustrate the difference between the rejected samples and the approved samples.
Specifically, we adopt the Lending Club dataset\footnote{\url{https://www.kaggle.com/wordsforthewise/lending-club}} in our analysis.
Lending Club\footnote{\url{https://www.lendingclub.com/}} is one of the largest credit loan companies worldwide.
Though there are no ground-truth default/non-default labels associated with the rejected samples, the Lending Club dataset is a valuable dataset since it is a rare publicly available credit scoring dataset that contains rejected samples and their features.
With the help of the Lending Club dataset, we can empirically investigate the difference between the approved samples and the rejected samples.

In the Lending Club dataset, there are totally four feature fields: Fico score, debt-to-income ratio, loan amount, and employment length.
In order to investigate the difference between the approved samples and the rejected samples, in Tab. \ref{tab:data_stydy}, we show Fico score, debt-to-income ratio, loan amount and employment length of the approved customers and the rejected customers in the Lending Club dataset across different years.
Meanwhile, we remove samples with missing values in the Lending Club dataset.
From results in Tab. \ref{tab:data_stydy}, we have following observations.
(1) Approved customers have higher Fico scores, while rejected customers have lower Fico scores. Fico score, which is provided by the Fico Company\footnote{\url{https://www.fico.com/}}, integrates a customer’s credit record. A larger Fico score means better credit history of a customer, which leads to lower default risk.
(2) Loan amounts are similar between approved and rejected customers. And approved customers have lower debt-to-income ratios, while rejected customers have higher debt-to-income ratios. Debt-to-income ratio is calculated as $\frac{Monthly Income}{Loan Amount}$, which means debt-to-income ratio is a derived variable of loan amount. A larger debt-to-income ratio usually means worse repayment ability, which leads to a larger probability of default.
(3) Approved customers have larger employment length than rejected customers. Customers with larger employment length usually have better repayment ability.
These observations tell us that, rejected loan applications are more likely to default or overdue.

\subsection{Motivation}

We show by Theorem \ref{theorem:correlation} that the prediction of default/non-default task and the prediction of rejection/approval task are positively correlated.
Therefore, the learning of rejection/approval task can be beneficial for the learning of default/non-default task via multi-task learning methods.

\begin{theorem}[Correlation between default/non-default and rejection/approval] \label{theorem:correlation}
Assume that default samples are denoted as $D$, non-default samples are denoted as $\bar{D}$, rejected samples are denoted as $R$, and approved samples are denoted as $\bar{R}$. A loan is either default or non-default, and is either rejected or approved. If the prerequisite rejection strategy is effective, which means that the default rate of rejected samples is indeed larger than the default rate of the approved samples, i.e., $P(D|R)>P(D|\bar{R})$, then the correlation coefficient $corr(D,R)=\frac{P(DR)-P(D)P(R)}{\sqrt{P(D)P(\bar{D})P(R)P(\bar{R})}}>0$, i.e., the rejection and the default are positively correlated.
\end{theorem}

\begin{proof}
If $P(D|R)>P(D|\bar{R})$, we have $P(D|R)>P(D|R+\bar{R})=P(DR+D\bar{R})P(R+\bar{R})=P(D)$ since that $R+\bar{R}$ is the full set and $P(R+\bar{R})=1$. Therefore, we have $P(DR)=P(D|R)P(R)>P(D)P(R)$, and $corr(D,R)=\frac{P(DR)-P(D)P(R)}{\sqrt{P(D)P(\bar{D})P(R)P(\bar{R})}}>0$, i.e., the learning of rejection and the learning of default are positively correlated.
\end{proof}

To be noted, in Theorem \ref{theorem:correlation}, we rely on the assumption that the prerequisite rejection strategy is effective, which means the credit approval system is useful.
By this assumption, we can obtain an obviously smaller default rate in approved samples than that in rejected samples.
This assumption is tenable in real-world applications, otherwise, credit loan institutions will face tremendous economic losses, so that the strategy is unlikely to be under using.

In summary, the learning of default/non-default (with limited and biased data) can benefit from the learning of rejection/approval (with a larger amount of unbiased data).

% 320 612
\begin{figure*}
	\centering
	\subfigure[RMT-Net under single strategy.]{
		\begin{minipage}[b]{0.33\textwidth}
			\includegraphics[width=1\textwidth]{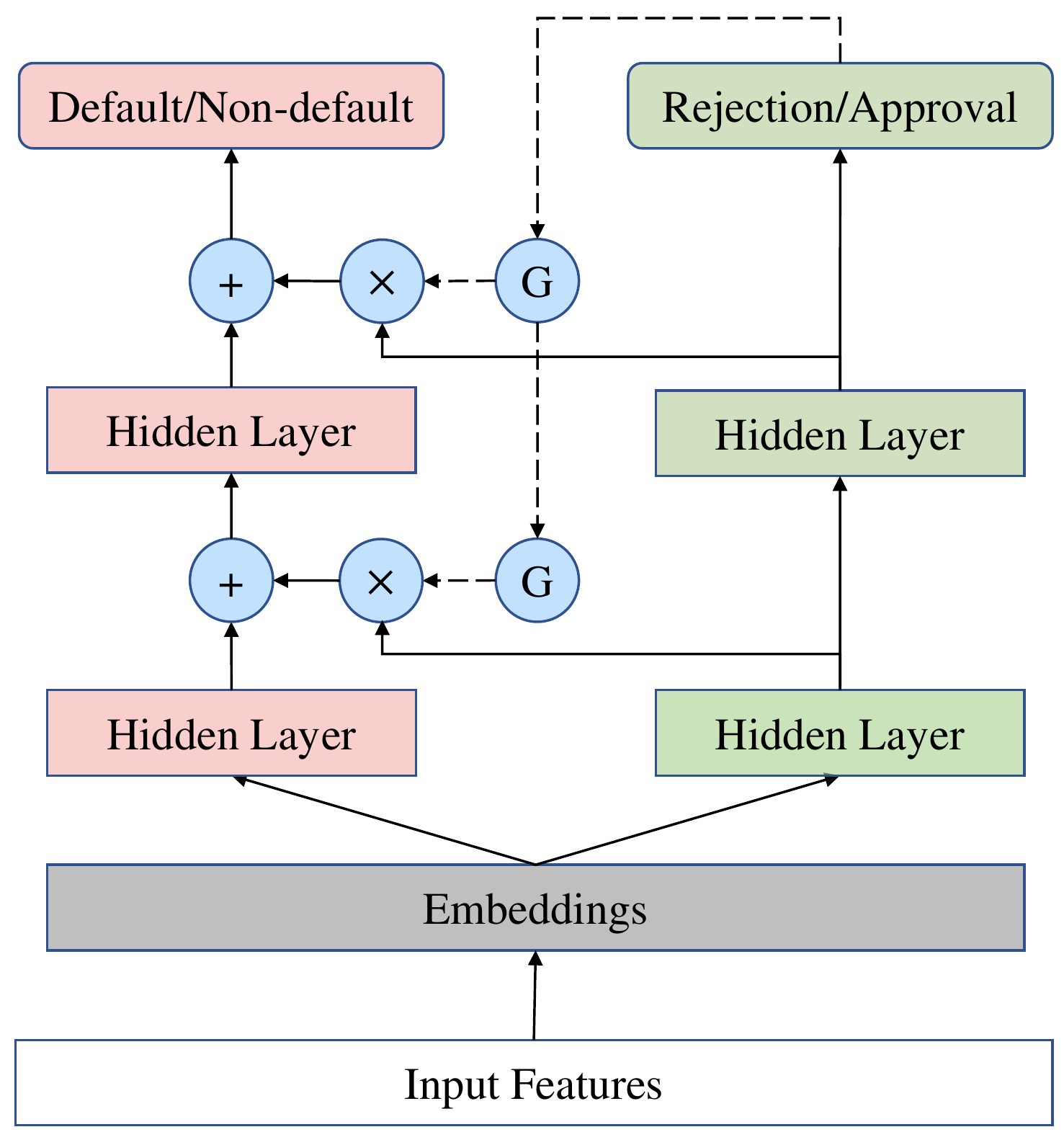}
		\end{minipage}
	    \label{fig:RMT}
	}
	\subfigure[RMT-Net++ under multiple strategies.]{
		\begin{minipage}[b]{0.631\textwidth}
			\includegraphics[width=1\textwidth]{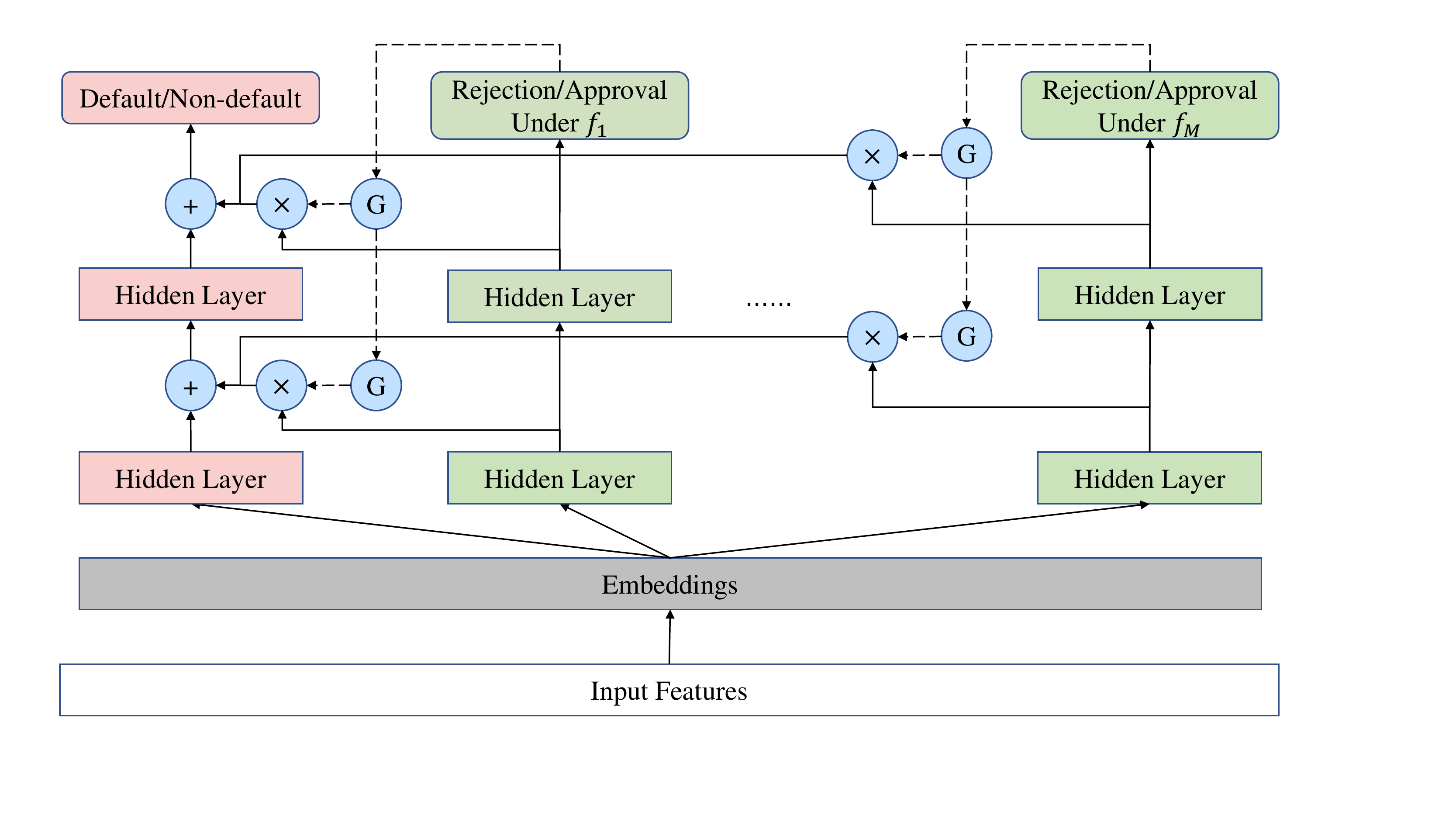}
		\end{minipage}
	    \label{fig:RMT++}
	}
	\caption{Schematic diagram of RMT-Net and RMT-Net++. ``G'' parts denote the gating networks for calculating the weights of the rejection/approval task. The solid lines indicate the feedforward of feature representations. The dotted lines indicate the message passing in gating networks for calculating task weights.}
	\label{fig:model}
\end{figure*}

\section{RMT-Net Under Single Rejection/Approval Strategy} \label{sec:rmt}

According to both data study and theoretical analysis in Sec. \ref{sec:analysis}, the default/non-default task and the rejection/approval task are highly correlated.
Thus, it is proper to model biased credit scoring data with multi-task learning.
In this section, we detail our proposed RMT-Net model under single strategy, which focuses on scenarios with a constant rejection/approval strategy.

\subsection{Notations} \label{sec:notation1}

All applications are denoted as a set of samples $\mathcal S = \left\{ {\mathop s\nolimits_1 ,\mathop s\nolimits_2 ,......,\mathop s\nolimits_N } \right\}$, where $N$ is the total number of loan applications. Each sample ${\mathop s\nolimits_i}$ is associated with a feature vector $\mathop x\nolimits_i  \in \mathop {\mathcal R}\nolimits^d$, where $d$ is the feature dimensionality.
For each sample $\mathop s\nolimits_i  \in \mathcal S$, we have a $\mathop r\nolimits_i  \in \left\{ {0,1} \right\}$ to indicate its rejection/approval label, where $\mathop r\nolimits_i  = 0$ means approval and $\mathop r\nolimits_i  = 1$ means rejection.
For each sample ${\mathop s\nolimits_i}$ approved, i.e., $\mathop r\nolimits_i  = 0$, we have ground-truth default/non-default label $\mathop y\nolimits_i  \in \left\{ {0,1} \right\}$, where $\mathop y\nolimits_i  = 1$ means default and $\mathop y\nolimits_i  = 0$ means non-default.
For each sample ${\mathop s\nolimits_i}$ rejected, i.e., $\mathop r\nolimits_i  = 1$, we have no ground-truth default/non-default label for model training.
We need to use the whole set of $\mathcal S$ to train a model, which should perform well on both approved and rejected samples for default loans prediction.

\subsection{Model Architecture}

Current state-of-the-art MTL approaches mainly focus on adaptively learning weights of different tasks in a mixture-of-experts structure, such as MMOE \cite{ma2018modeling} and PLE \cite{tang2020progressive}.
In such MTL structures, task weights change in different samples, which makes useful but not conflicting information to adaptively share among tasks.
Considering these MTL approaches have achieved promising performances in various scenarios, a simple way for reject inference might be directly applying them in training credit scoring models.
However, according to experiments in Sec. \ref{sec:comparison}, we obtain very poor performances in default prediction on rejected samples.
This could be because some task weights are not well learned, since we have no observed default/non-default labels for rejected samples during model training.
In the feature distribution of rejected samples, there is no supervision for optimizing the task weights to control how much information is shared from the rejection/approval task to the default/non-default task.
Thus, existing MTL approaches fail in biased credit scoring data, and we propose RMT-Net to learn the task weights by a gating network based on rejection probabilities.

Our proposed RMT-Net consists of 1) an embedding layer that learns the dense representation of the feature vectors; 2) a multi-layer rejection/approval prediction network (R/A-Net); 3) a multi-layer default/non-default prediction network (D/N-Net); 4) a gating network that learns the weights of the rejection/approval task for the default/non-default task.
The model architecture is shown in Fig. \ref{fig:RMT}.

\subsection{Embedding Layer}

The embedding layer transforms each feature into a dense vector to facilitate learning efficiency.
For numerical features that are infeasible to embed, feature discretization techniques \cite{liu2002discretization}\cite{franc2018learning}\cite{liu2020empirical} are conducted, which are proven useful to improve the learning efficiency \cite{chapelle2014simple}.
The embedding layer converts each feature vector $\mathop x\nolimits_i  \in \mathop {\mathcal R}\nolimits^d$ into an embedded representation $\mathop e\nolimits_i  \in \mathop {\mathcal R}\nolimits^{d \times k}$, where $k$ is the embedding dimensionality.
Afterward, the embedded representation is flattened to form a one-dimensional embedded vector $\mathop e\nolimits_i  \in \mathop {\mathcal R}\nolimits^{dk}$.
Therefore, after embedding, the data samples $\mathcal S \in \mathop {\mathcal R}\nolimits^{N \times d}$ will become $E \in \mathop {\mathcal R}\nolimits^{N \times dk}$.

\subsection{Rejection/Approval Prediction Network} \label{sec:RA}

The rejection/approval prediction network (R/A-Net) has multiple linear layers with an activation function. 
The first layer multiplies the embedded vector $\mathop e\nolimits_i \in \mathop {\mathcal R}\nolimits^{dk}$ with a weight term $\mathop w\nolimits_R^{(1)} \in \mathop {\mathcal R}\nolimits^{dk \times h_1}$, adds a bias term $\mathop b\nolimits_R^{(1)} \in \mathop {\mathcal R}\nolimits^{h_1}$, and activates with a nonlinear function $relu$. The first layer results in the latent representation of the first layer $\mathop p\nolimits_i^{(1)} \in \mathop {\mathcal R}\nolimits^{h_1}$, where $h_j$ means the dimensionality of the $j$-th layer. 
Except the final layer, the other layers further update the latent representations in the same way of the first layer as
\begin{equation}
\mathop p\nolimits_i^{(j)} = relu(\mathop p\nolimits_i^{(j-1)} \mathop w\nolimits_R^{(j)} + \mathop b_R\nolimits^{(j)}),
\end{equation}
where $\mathop p\nolimits_i^{(j)} \in \mathop {\mathcal R}\nolimits^{h_j}$.
For the final layer, if we denote the number of layers as $t$, we will have $h_t=1$. The final output of the R/A-Net will be the rejection probability
\begin{equation}
\mathop p\nolimits_i^{(t)} = \sigma( \mathop p\nolimits_i^{(t-1)} \mathop w_R\nolimits^{(t)} + \mathop b_R\nolimits^{(t)}),
\end{equation}
where $\mathop p\nolimits_i^{(t)} \in \mathop {\mathcal R}\nolimits^{1}$, and $\sigma (\cdot)$ is the sigmoid function.

\subsection{Default/Non-Default Prediction Network}

The default/non-default prediction network (D/N-Net) has the same number of layers of the R/A-Net as $t$ and resembles its basic linear computations with activation.
The difference is that the latent representation is determined under the proposed multi-task learning framework.
We first design a gating network at each layer $j$ that uses the output of R/A-Net as
\begin{equation}
\mathop g\nolimits_i^{(j)} = \sigma( \alpha^{(j)} \mathop p\nolimits_i^{(t)} + \beta^{(j)} ),
\end{equation}
where $\mathop g\nolimits_i^{(j)} \in \mathop {\mathcal R}\nolimits^{1}$.
$\alpha^{(j)} \in \mathop {\mathcal R}\nolimits^{1}$ and $\beta^{(j)} \in \mathop {\mathcal R}\nolimits^{1}$ are learnable parameters to control the value of $g$, which is designed to indicate the ratio of learning the rejection/approval task for the learning of default/non-default.
Apparently, the activated output of the R/A-Net, i.e. the probability of rejection, determines how much information in R/A-Net are used to learn the default/non-default task.
Except the final layer, the latent representation $\mathop q\nolimits_i^{(j)} \in \mathop {\mathcal R}\nolimits^{h_j}$ of each layer of the D/N-Net can be denoted as
\begin{equation}
\mathop q\nolimits_i^{(j)} = relu(\mathop q\nolimits_i^{(j-1)} \mathop w\nolimits_D^{(j)} + \mathop b\nolimits_D^{(j)}) + \mathop g\nolimits_i^{(j)} \mathop p\nolimits_i^{(j)},
\end{equation}
where $\mathop w\nolimits_D^{(j)}$ and $\mathop b\nolimits_D^{(j)}$ denote the weight term and the bias term respectively in the D/N-Net.
For the final layer, the final output of the D/N-Net is denoted as
\begin{equation}
\mathop q\nolimits_i^{(t)} = \sigma( \mathop q\nolimits_i^{(t-1)} \mathop w\nolimits_D^{(t)} + \mathop b\nolimits_D^{(t)} ),
\end{equation}
where $\mathop q\nolimits_i^{(t)} \in \mathop {\mathcal R}\nolimits^{1}$ means the default probability.

In this way, we can adaptively control the weights of the rejection/approval task in different samples by the rejection probability, and overcome the under-fitting problem of conventional MTL approaches in the feature distribution of rejected samples.
In our proposed MTL architecture, the model can learn that, with a larger rejection probability, less reliable information can be learned in R/A-Net, and more information should be shared from D/N-Net.

\subsection{Loss Function}
For the rejection/approval task, given the rejection/approval label $r_i$ and the output of the R/A-Net $\mathop p\nolimits_i^{(t)}$, we have
\begin{equation}
\mathcal{L}_1=-\sum_{i=1}^N \mathop p\nolimits_i^{(t)} log (r_i) + (1- \mathop p\nolimits_i^{(t)}) log (1-r_i).
\end{equation}

For the default/non-default prediction task, considering we have no observed default/non-default labels for rejected samples, we need to mask the loss with reject/approval labels.
Given the default/non-default label $y_i$, the rejection/approval label $r_i$, and the output of the D/N-Net $\mathop q\nolimits_i^{(t)}$, we have
\begin{equation}
\mathcal{L}_2 = -\sum_{i=1}^N (\mathop q\nolimits_i^{(t)} log (y_i) + (1- \mathop q\nolimits_i^{(t)}) log (1-y_i)) (1-r_i).
\end{equation}

If we define $\eta$ as a hyperparameter to balance the two losses, the overall loss function is denoted as
\begin{equation}
\mathcal{L} = (1-\eta) \cdot \mathcal{L}_1 + \eta \cdot \mathcal{L}_2.
\end{equation}

\section{RMT-Net++ Under Multiple Rejection/Approval Strategies}

In Sec. \ref{sec:rmt}, we have detailed the RMT-Net model under single constant rejection/approval strategy.
However, in real-world applications, rejection/approval strategies change frequently, so we usually have multiple strategies in different periods.
For example, financial institutions may modify approval ratios, add important factors or new factors in the credit evaluation systems.
Thus we encounter a variety of rejection/approval segmentation in the training data of credit scoring.
If we simply regard these strategies as one strategy, there will be conflicts in the model when classifying approved and rejected samples.
Therefore, we extend RMT-Net to RMT-Net++, which further incorporates multiple strategies in the multi-task learning framework to improve the prediction accuracy.

\subsection{Notations}
Basic notations still follow the notations in Sec. \ref{sec:notation1}.
Besides, we have $M$ different rejection/approval strategies $\left\{ {\mathop f\nolimits_1 ,\mathop f\nolimits_2 ,......,\mathop f\nolimits_M } \right\}$.
Each sample $\mathop s\nolimits_i  \in \mathcal S$ is under one specific rejection/approval strategy $f_{{s_i}}$, and its rejection/approval label $\mathop r\nolimits_i  \in \left\{ {0,1} \right\}$ is determined by the corresponding strategy.
Multiple-strategy scenarios can degenerate to single-strategy scenarios when $M=1$.

\subsection{Model Architecture}
The model architecture of RMT-Net++ is based on that of RMT-Net. RMT-Net++ has the same embedding layer as the RMT-Net. Because RMT-Net assumes $M=1$, i.e., it only uses a single rejection/approval strategy for the learning of default/non-default and contains one rejection/approval prediction network (R/A-Net). RMT-Net++, on the contrary, has $M \neq 1$ R/A-Nets that each learns the data samples whose labels are collected by its unique strategy. Therefore, the gating network and the default/non-default prediction network (D/N-Net) in the RMT-Net++ will change accordingly.
The model architecture is shown in Fig. \ref{fig:RMT++}. 

\subsection{Rejection/Approval Networks++}
The rejection/approval networks in RMT-Net++ (R/A-Nets++) include $M \neq 1$ rejection/approval prediction networks (R/A-Net) that each only learns on data samples of the corresponding rejection/approval strategy. Here, each R/A-Net is the same as used in RMT-Net, as described in Sec. \ref{sec:RA}. To distinguish among different R/A-Nets, we might as well add a square bracket to the lower right corner of the notations of each R/A-Net.
In this case, for the $m$-th R/A-Net, we can denote the latent representations of its layers as
\begin{equation}
\mathop p\nolimits_{i, [m]}^{(j)} = relu(\mathop p\nolimits_{i, [m]}^{(j-1)} \mathop w\nolimits_{R,[m]}^{(j)} + \mathop b\nolimits_{R,[m]}^{(j)}),
\end{equation}
and denote the final outputs as
\begin{equation}
\mathop p\nolimits_{i, [m]}^{(t)} = \sigma( \mathop p\nolimits_{i, [m]}^{(t-1)} \mathop w\nolimits_{R,[m]}^{(t)} + \mathop b\nolimits_{R,[m]}^{(t)} ),
\end{equation}
where $\mathop p\nolimits_{i, [m]}^{(j)} \in \mathop {\mathcal R}\nolimits^{h_j}$ and $\mathop p\nolimits_{i, [m]}^{(t)} \in \mathop {\mathcal R}\nolimits^{1}$.

\subsection{Default/Non-Default Network++}

The default/non-default network (D/N-Net++) has the same number of layers of each R/A-Net++ as $t$. Its latent representation is calculated under the proposed multi-task learning framework with multiple strategies.
We design a gating network at each layer $j$ that uses the output of the $m$-th R/A-Net++ as
\begin{equation}
\mathop g\nolimits_{i, [m]}^{(j)} = \sigma( \alpha_{[m]}^{(j)} \mathop p\nolimits_{i, [m]}^{(t)} + \beta_{[m]}^{(j)} ),
\end{equation}
where $\mathop g\nolimits_{i, [m]}^{(j)} \in \mathop {\mathcal R}\nolimits^{1}$.
$\alpha_{[m]}^{(j)} \in \mathop {\mathcal R}\nolimits^{1}$ and $\beta_{[m]}^{(j)} \in \mathop {\mathcal R}\nolimits^{1}$ are learnable parameters that control the value of $g_{[m]}$, which are designed to indicate the ratio of learning each rejection/approval task under strategy $f_m$ for the learning of default/non-default.
Here, we consider the activated output of each R/A-Net++, i.e. the probability of rejection by each strategy, determines how much parameters of each R/A-Net++ are used for the default/non-default task.
Except the final layer, the latent representation of each layer of the D/N-Net++ can be denoted as
\begin{equation}
\mathop q\nolimits_i^{(j)} = relu(\mathop q\nolimits_i^{(j-1)} \mathop w\nolimits_D^{(j)} + \mathop b\nolimits_D^{(j)}) + \sum_{m=1}^{M} \mathop g\nolimits_{i, [m]}^{(j)} \mathop p\nolimits_{i, [m]}^{(j)},
\end{equation}
where $\mathop q\nolimits_i^{(j)} \in \mathop {\mathcal R}\nolimits^{h_j}$.
And the final output is denoted as
\begin{equation}
\mathop q\nolimits_i^{(t)} = \sigma( \mathop q\nolimits_i^{(t-1)} \mathop w\nolimits_D^{(t)} + \mathop b\nolimits_D^{(t)} ),
\end{equation}
where $\mathop q\nolimits_i^{(t)} \in \mathop {\mathcal R}\nolimits^{1}$ means the default probability.

\subsection{Loss Function}
For the rejection/approval task, given the rejection/approval label $r_i$ and the outputs of the R/A-Net++, we have
\begin{equation}
\small
% \begin{array}{l}
\mathcal{L}_1 =  - \sum\limits_{m = 1}^M {\sum\limits_{i = 1,{f_{{s_i}}} = {f_m}}^N {p_{i,[m]}^{(t)}} } log({r_i}) + (1 - p_{i,[m]}^{(t)})log(1 - {r_i}).
% \end{array}
\end{equation}

For the default/non-default prediction task, given the default/non-default label $y_i$, the rejection/approval label $r_i$, and the output of the D/N-Net++ $\mathop q\nolimits_i^{(t)}$, we have
\begin{equation}
\mathcal{L}_2 = -\sum_{i=1}^N (\mathop q\nolimits_i^{(t)} log (y_i) + (1- \mathop q\nolimits_i^{(t)}) log (1-y_i)) (1-r_i).
\end{equation}

With $\eta$ as a hyperparameter to balance the two losses, we denote the overall loss function as
\begin{equation}
\mathcal{L} = (1-\eta) \cdot \mathcal{L}_1  / M + \eta \cdot \mathcal{L}_2.
\end{equation}

\section{Experiments}
In this section, we conduct empirical experiments to verify the effectiveness of RMT-Net and RMT-NET++.

\subsection{Datasets}

In our experiments, we are going to verify our proposed approaches from three aspects: experiments on approved samples, experiments on both approved and rejected samples, experiments under multiple rejection/approval policies.

\begin{table}[t]
  \centering
  \caption{Details of datasets with single rejection/approval policy.}
    \begin{tabular}{cccc}
    \toprule
    \multirow{2}[2]{*}{Dataset} & \multirow{2}[2]{*}{Rejection Ratio (\%)} & \multicolumn{2}{c}{Default Ratio (\%)} \\
          &       & Approved & Rejected \\
    \midrule
    Lending1 & 84.00  & 15.40  & - \\
    Lending2 & 87.63  & 17.36  & - \\
    Lending3 & 54.86  & 18.08  & - \\
    Home1 & 75.00  & 1.97  & 10.06  \\
    Home2 & 75.00  & 2.85  & 9.69  \\
    PPD1  & 75.00  & 4.40  & 15.90  \\
    PPD2  & 75.00  & 2.20  & 9.04  \\
    \bottomrule
    \end{tabular}%
  \label{tab:data1}%
\end{table}%

\begin{table}[t]
\caption{Details of datasets with multiple rejection/approval policies.}
\centering
 
\subtable[The Lending-M dataset.]{
    \begin{tabular}{cccc}
    \toprule
    \multirow{2}[2]{*}{Policy} & \multirow{2}[2]{*}{Rejection Ratio (\%)} & \multicolumn{2}{c}{Default Ratio (\%)} \\
          &       & Approved & Rejected \\
    \midrule
    1     & 84.00  & 15.40  & - \\
    2     & 87.63  & 17.36  & - \\
    3     & 54.86  & 18.08  & - \\
    \bottomrule
    \end{tabular}%
}
 
\qquad
 
\subtable[The Home-M dataset.]{
    \begin{tabular}{cccc}
    \toprule
    \multirow{2}[2]{*}{Policy} & \multirow{2}[2]{*}{Rejection Ratio (\%)} & \multicolumn{2}{c}{Default Ratio (\%)} \\
          &       & Approved & Rejected \\
    \midrule
    1     & 75.00  & 3.54  & 9.45  \\
    2     & 75.00  & 2.60  & 9.84  \\
    \bottomrule
    \end{tabular}%
}
 
\qquad
 
\subtable[The PPD-M dataset.]{
    \begin{tabular}{cccc}
    \toprule
    \multirow{2}[2]{*}{Policy} & \multirow{2}[2]{*}{Rejection Ratio (\%)} & \multicolumn{2}{c}{Default Ratio (\%)} \\
          &       & Approved & Rejected \\
    \midrule
    1     & 75.00  & 2.92  & 8.63  \\
    2     & 75.00  & 3.56  & 8.53  \\
    \bottomrule
    \end{tabular}%
}
\label{tab:data2}%
\end{table}

\begin{table*}[t]
  \centering
  \caption{Performance comparison on approval-only datasets, evaluated by AUC (\%) and KS(\%). Average values are also listed. $*$ denotes statistically significant improvement, measured by t-test with p-value$<0.01$, over the second-best approach on each dataset.}
    \begin{tabular}{cc|cccccccc}
    \toprule
    \multirow{2}[2]{*}{Type} & \multirow{2}[2]{*}{Approach} & \multicolumn{2}{c}{Lending1} & \multicolumn{2}{c}{Lending2} & \multicolumn{2}{c}{Lending3} & \multicolumn{2}{c}{average} \\
          &       & AUC   & KS    & AUC   & KS    & AUC   & KS    & AUC   & KS \\
    \midrule
          & LR    & 59.94  & 13.53  & 60.52  & 15.81  & 61.67  & 17.43  & 60.71  & 15.59  \\
    Baseline & MLP   & 59.96  & 13.77  & 60.50  & 15.72  & 61.54  & 17.53  & 60.67  & 15.67  \\
          & XGB   & 59.93  & 13.62  & 60.54  & 15.69  & 61.34  & 17.18  & 60.60  & 15.50  \\
    \midrule
          & ST+LR & 60.18  & 14.02  & 60.49  & 15.73  & 61.87  & 17.86  & 60.85  & 15.87  \\
    Semi-Supervised & ST+MLP & 60.13  & 13.89  & 60.54  & 15.89  & 61.71  & 17.78  & 60.79  & 15.85  \\
    Learning & ST+XGB & 60.19  & 14.11  & 60.51  & 15.86  & 61.83  & 17.89  & 60.84  & 15.95  \\
          & SS-GMM & 60.09  & 13.75  & 60.55  & 15.78  & 62.06  & 18.24  & 60.90  & 15.92  \\
    \midrule
          & IPS+LR & 60.21  & 13.45  & 60.36  & 15.52  & 61.53  & 17.59  & 60.70  & 15.52  \\
          & IPS+MLP & 60.25  & 14.58  & 60.44  & 16.00  & 61.59  & 17.63  & 60.76  & 16.07  \\
          & DR+LR & 60.18  & 13.69  & 60.32  & 15.36  & 61.37  & 17.41  & 60.62  & 15.49  \\
          & DR+MLP & 60.22  & 14.34  & 60.48  & 15.72  & 61.68  & 17.72  & 60.79  & 15.93  \\
    Counterfactual & DRJL+LR & 60.28  & 14.53  & 60.41  & 15.52  & 61.42  & 17.58  & 60.70  & 15.88  \\
    Leaning & DRJL+MLP & 60.25  & 14.59  & 60.54  & 16.10  & 61.73  & 17.81  & 60.84  & 16.17  \\
          & ACL+LR & 60.31  & 14.46  & 60.49  & 15.75  & 61.79  & 17.88  & 60.86  & 16.03  \\
          & ACL+MLP & 60.34  & 14.63  & 60.59  & 15.86  & 61.85  & 17.94  & 60.93  & 16.14  \\
          & SRDO+LR & 60.24  & 14.03  & 60.56  & 15.74  & 61.58  & 17.59  & 60.79  & 15.79  \\
          & SRDO+MLP & 60.13  & 13.59  & 60.39  & 15.82  & 61.47  & 17.47  & 60.66  & 15.63  \\
    \midrule
          & Cross-Stitch & 60.33  & 14.51  & 60.76  & 16.54  & 62.17  & 18.56  & 61.09  & 16.54  \\
    Multi-Task & MMOE  & 60.24  & 14.63  & 60.62  & 16.17  & 62.09  & 18.31  & 60.98  & 16.37  \\
    Learning & PLE   & 60.32  & 14.41  & 60.71  & 16.45  & 61.97  & 18.04  & 61.00  & 16.30  \\
          & RMT-Net & $\ \ $\textbf{60.61}$^*$ & $\ \ $\textbf{15.35}$^*$ & $\ \ $\textbf{61.02}$^*$ & $\ \ $\textbf{17.08}$^*$ & $\ \ $\textbf{62.48}$^*$ & $\ \ $\textbf{18.97}$^*$ & $\ \ $\textbf{61.37}$^*$ & $\ \ $\textbf{17.13}$^*$ \\
    \bottomrule
    \end{tabular}%
  \label{tab:comparison1}%
\end{table*}%

\textbf{Approval-only datasets}:
Firstly, we need to investigate whether our proposed approaches can improve the performances on approved samples.
This means, training set and testing set, which are both approved samples, share similar data distributions.
We select datasets with real rejected samples, but without ground-truth default/non-default labels for rejected samples.
From the Lending Club dataset, we extract samples in 2013, 2014 and 2015 to construct the \textbf{Lending1} dataset, the \textbf{Lending2} dataset and the \textbf{Lengding3} dataset respectively.
% The reason we only select samples in 2013 and 2014 is that the rejection ratios in these two years are much higher, so the data bias in the training set is bigger.
% In this way, performance comparison among different approaches will be clearer.
In our experiments, we randomly use $60\%$, $20\%$, and $20\%$ approved samples as training, validation, and testing set respectively.
Features of rejected samples can be used during the training of credit scoring models.

\textbf{Approval-rejection datasets}:
Secondly, we investigate whether our proposed approaches can stably achieve promising performances in different data distributions.
This requires us to conduct experiments on both approved and rejected samples.
However, ground-truth default/non-default labels for real rejected samples are hard to obtain, so that we need to generate synthetic rejected samples from some real-world credit scoring datasets.
Here, we incorporate the Home\footnote{\url{https://www.kaggle.com/c/home-credit-default-risk}} and PPD\footnote{\url{https://www.kesci.com/home/competition/56cd5f02b89b5bd026cb39c9/content/1}} datasets.
To be noted, samples in these two datasets are actually all approved samples with ground-truth default/non-default labels.
The process of generating synthetic rejected samples is conducted as follows:
(1) use random $1/3$ samples (denoted initial samples) and $\varepsilon \cdot d$ features to train an LR model as the synthetic rejection/approval policy of the credit scoring system, where $\varepsilon$ is a ratio that $0\% < \varepsilon \le 100\%$, and $d$ is the dimensionality of input features;
(2) use the trained LR model to predict the default probabilities of the rest $2/3$ samples (denoted as main samples);
(3) assign $3/4$ of the main samples with largest default probabilities as synthetic rejected samples;
(4) assign $1/4$ of the main samples with smallest default probabilities as approved samples.
This process is very similar to a real-world credit scoring system, which uses some initial loan applications to train a machine learning model for future decisions of rejection/approval.
$\varepsilon$ is to control the strength of rejection/approval.
With larger $\varepsilon$ values, more features will be used, and data distributions between rejected and approved samples will be more distinguishable.
On each of Home and PPD, we run the above synthetic process two times and set $\varepsilon=100\%$ and $\varepsilon=50\%$ respectively.
This results in four approval-rejection datasets: \textbf{Home1}, \textbf{Home2}, \textbf{PPD1} and \textbf{PPD2}.
On each dataset, we randomly use $60\%$, $20\%$, and $20\%$ approved samples as training, validation, and testing set respectively.
Meanwhile, the testing set also contains all the rejected samples to conduct performance comparison across different data distributions.
Moreover, features of rejected samples are also used during the training of credit scoring models, but the ground-truth default-non-default labels cannot be used during training.

\textbf{Multi-policy datasets}:
Thirdly, we need to verify the effectiveness of our proposed approaches under multiple policies.
In the Lending Club dataset, we regard the policies in 2013, 2014 and 2015 as three different policies and obtain a multi-policy dataset named \textbf{Lending-M}.
For Home and PPD, we split each dataset into two equal subsets, and run the synthetic rejected sample generation process on each sub-set with $\varepsilon=50\%$.
This results in two different rejection/approval policies on each dataset.
Thus, we obtain two multi-policy datasets named \textbf{Home-M} and \textbf{PPD-M}.
The training/validation/testing split of Lending-M is the same as that of previous approval-only datasets.
The training/validation/testing split of Home-M and PPD-M is the same as that of previous approval-rejection datasets.
That is to say, there are only approved samples in Lending-M, and both approved and synthetic rejected samples in Home-M and PPD-M.

\begin{table*}[t]
  \centering
  \caption{Performance comparison on approval-rejection datasets, evaluated by AUC (\%) and KS(\%). Average values are also listed. $*$ denotes statistically significant improvement, measured by t-test with p-value$<0.01$, over the second-best approach on each dataset.}
    \begin{tabular}{cc|cccccccccc}
    \toprule
    \multirow{2}[2]{*}{Type} & \multirow{2}[2]{*}{Approach} & \multicolumn{2}{c}{Home1} & \multicolumn{2}{c}{Home2} & \multicolumn{2}{c}{PPD1} & \multicolumn{2}{c}{PPD2} & \multicolumn{2}{c}{Average} \\
          &       & AUC   & KS    & AUC   & KS    & AUC   & KS    & AUC   & KS    & AUC   & KS \\
    \midrule
          & LR    & 54.87  & 7.42  & 68.05  & 26.46  & 62.83  & 19.88  & 58.80  & 13.41  & 61.14  & 16.79  \\
    Baseline & MLP   & 55.17  & 7.63  & 67.72  & 25.82  & 60.80  & 16.54  & 58.63  & 12.75  & 60.58  & 15.69  \\
          & XGB   & 55.60  & 8.12  & 67.81  & 26.07  & 63.21  & 20.93  & 59.32  & 14.19  & 61.49  & 17.33  \\
    \midrule
          & ST+LR & 66.12  & 23.96  & 67.88  & 27.71  & 66.22  & 23.90  & 64.57  & 21.11  & 66.20  & 24.17  \\
    Semi-Supervised & ST+MLP & 66.37  & 24.40  & 67.93  & 27.56  & 64.35  & 21.28  & 64.87  & 21.40  & 65.88  & 23.66  \\
    Learning & ST+XGB & 66.60  & 24.58  & 67.77  & 27.23  & 66.58  & 24.71  & 64.79  & 21.22  & 66.44  & 24.44  \\
          & SS-GMM & 66.21  & 23.61  & 68.59  & 27.71  & 67.12  & 25.93  & 64.50  & 20.96  & 66.61  & 24.55  \\
    \midrule
          & IPS+LR & 66.26  & 24.43  & 67.24  & 24.99  & 69.20  & 28.43  & 63.73  & 19.98  & 66.61  & 24.46  \\
          & IPS+MLP & 66.37  & 24.59  & 67.93  & 25.87  & 69.88  & 29.32  & 63.87  & 20.26  & 67.01  & 25.01  \\
          & DR+LR & 65.72  & 24.14  & 67.66  & 25.23  & 69.66  & 29.64  & 63.49  & 19.95  & 66.63  & 24.74  \\
          & DR+MLP & 65.87  & 24.20  & 67.92  & 25.93  & 69.43  & 29.39  & 64.26  & 21.19  & 66.87  & 25.18  \\
    Counterfactual & DRJL+LR & 66.67  & 24.61  & 68.42  & 27.11  & 69.31  & 29.10  & 64.20  & 21.27  & 67.15  & 25.52  \\
    Leaning & DRJL+MLP & 66.30  & 24.42  & 68.68  & 27.97  & 69.57  & 29.71  & 64.59  & 21.58  & 67.29  & 25.92  \\
          & ACL+LR & 65.97  & 23.56  & 67.81  & 25.09  & 68.49  & 27.54  & 62.86  & 19.11  & 66.28  & 23.83  \\
          & ACL+MLP & 66.59  & 24.23  & 67.91  & 25.40  & 69.30  & 27.69  & 63.69  & 19.77  & 66.87  & 24.27  \\
          & SRDO+LR & 66.14  & 24.53  & 68.26  & 26.13  & 69.47  & 28.60  & 63.41  & 19.74  & 66.82  & 24.75  \\
          & SRDO+MLP & 65.98  & 24.20  & 67.58  & 25.49  & 69.62  & 28.49  & 63.66  & 19.90  & 66.71  & 24.52  \\
    \midrule
          & Cross-Stitch & 64.33  & 21.96  & 66.96  & 23.81  & 66.71  & 26.63  & 60.93  & 17.06  & 64.73  & 22.37  \\
    Multi-Task & MMOE  & 56.31  & 14.39  & 63.55  & 22.49  & 66.09  & 24.53  & 56.17  & 10.54  & 60.53  & 17.99  \\
    Learning & PLE   & 57.63  & 15.70  & 63.90  & 21.51  & 65.89  & 24.61  & 54.62  & 9.16  & 60.51  & 17.75  \\
          & RMT-Net & $\ \ $\textbf{71.03}$^*$ & $\ \ $\textbf{30.84}$^*$ & $\ \ $\textbf{71.99}$^*$ & $\ \ $\textbf{32.40}$^*$ & $\ \ $\textbf{71.00}$^*$ & $\ \ $\textbf{30.30}$^*$ & $\ \ $\textbf{69.44}$^*$ & $\ \ $\textbf{29.00}$^*$ & $\ \ $\textbf{70.87}$^*$ & $\ \ $\textbf{30.64}$^*$ \\
    \bottomrule
    \end{tabular}%
  \label{tab:comparison2}%
\end{table*}%

Approval-only datasets and approval-rejection datasets are both datasets with a single rejection/approval policy, and their details are shown in Tab. \ref{tab:data1}.
And details of multi-policy datasets are illustrated in Tab. \ref{tab:data2}.
As in real systems, features in above datasets are constructed based on records before the time of each loan application.
This makes us available to predict whether a customer will default if he or she gets the loan, for we need to know the default probability before the loan approval.
If a customer applies for loan for more than one time, there will exist multiple samples in the data, and each of them corresponds to one application, with features constructed based on records before the corresponding application time.

\subsection{Settings}

We compare four types of approaches: baselines, semi-supervised learning approaches, counterfactual learning approaches and multi-task learning approaches.

For single-policy datasets, i.e., approval-only datasets and approval-rejection datasets, the following approaches are compared.
Baselines consists of three commonly-used classifiers for credit scoring: \textbf{LR}, \textbf{MLP} and \textbf{XGB}.
Among semi-supervised approaches, Self-Training (\textbf{ST}) \cite{maldonado2010semi} is incorporated with LR, MLP and XGB.
Another typical semi-supervised reject inference approach \textbf{SS-GMM} \cite{mancisidor2020deep} is also compared.
For counterfactual approaches, we involve \textbf{IPS} \cite{schnabel2016recommendations}\cite{swaminathan2015self}, \textbf{DR} \cite{jiang2016doubly}, \textbf{DRJL} \cite{wang2019doubly}, \textbf{ACL} \cite{xu2020adversarial} and \textbf{SRDO} \cite{shen2020stable}, and incorporate them with LR and MLP.
For multi-task approaches, besides our proposed \textbf{RMT-Net}, we involve \textbf{Cross-Stitch} \cite{misra2016cross}, \textbf{MMOE} \cite{ma2018modeling} and \textbf{PLE} \cite{tang2020progressive}.

For multi-policy datasets, we use the same  baselines, semi-supervised approaches and counterfactual approaches as above.
The only difference is that we consider multiple policies in multi-task approaches.
Specifically, we adjust Cross-Stitch, MMOE and PLE under multiple policies, and name them as \textbf{Cross-Stitch-M}, \textbf{MMOE-M} and \textbf{PLE-M}.
Moreover, our proposed \textbf{RMT-Net++} is also compared.

For our proposed RMT-Net and RMT-Net++, we empirically set learning rate as $0.001$, embedding dimensionality of each input feature as $4$, and dimensionality of hidden layers as $16$.
Meanwhile, according to the best performances on the validation set, we tune the loss balancing parameter $\lambda$ in the range of $[0.1,0.2,0.3,0.4,0.5]$, and the layer number $t$ in the range of $[2,3,4]$.
For other compared approaches, their hyper-parameters are tuned according to the performances on the validation set.
For all compared approaches including RMT-Net and RMT-Net++, early-stopping is conducted according to the performances on the validation set.
We run each approach $10$ times and report the median values.

We evaluate the performances on testing sets in terms of two commonly-used metrics for credit scoring and default prediction: \textbf{AUC} and \textbf{KS}.
AUC measures the overall ranking performance.
KS measures the largest difference between true positive rate and false positive rate on the ROC curve, which can be the threshold for loan approval.

\begin{table*}[t]
  \centering
  \caption{Performance comparison on multi-policy datasets, evaluated by AUC (\%) and KS(\%). Average values are also listed. $*$ denotes statistically significant improvement, measured by t-test with p-value$<0.01$, over the second-best approach on each dataset.}
    \begin{tabular}{cc|cccccccc}
    \toprule
    \multirow{2}[2]{*}{Type} & \multirow{2}[2]{*}{Approach} & \multicolumn{2}{c}{Lending-M} & \multicolumn{2}{c}{Home-M} & \multicolumn{2}{c}{PPD-M} & \multicolumn{2}{c}{Average} \\
          &       & AUC   & KS    & AUC   & KS    & AUC   & KS    & AUC   & KS \\
    \midrule
          & LR    & 60.66  & 15.43  & 70.31  & 29.47  & 62.66  & 18.83  & 64.54  & 21.24  \\
    Baseline & MLP   & 60.59  & 15.48  & 70.03  & 29.20  & 62.44  & 19.27  & 64.35  & 21.32  \\
          & XGB   & 60.69  & 15.53  & 70.11  & 29.32  & 62.76  & 19.54  & 64.52  & 21.46  \\
    \midrule
          & ST+LR & 60.59  & 15.47  & 70.53  & 30.29  & 65.66  & 23.89  & 65.59  & 23.22  \\
    Semi-Supervised & ST+MLP & 60.76  & 15.51  & 70.28  & 29.67  & 66.26  & 24.68  & 65.77  & 23.29  \\
    Learning & ST+XGB & 60.63  & 15.42  & 70.11  & 29.39  & 65.89  & 24.11  & 65.54  & 22.97  \\
          & SS-GMM & 60.81  & 15.70  & 70.47  & 29.88  & 66.08  & 24.57  & 65.79  & 23.38  \\
    \midrule
          & IPS+LR & 60.74  & 15.76  & 68.36  & 27.01  & 61.89  & 18.28  & 63.66  & 20.35  \\
          & IPS+MLP & 60.82  & 15.68  & 70.12  & 29.65  & 62.97  & 19.42  & 64.64  & 21.58  \\
          & DR+LR & 60.71  & 15.55  & 67.81  & 26.50  & 65.23  & 23.59  & 64.58  & 21.88  \\
          & DR+MLP & 60.78  & 15.59  & 69.72  & 28.96  & 65.76  & 24.13  & 65.42  & 22.89  \\
    Counterfactual & DRJL+LR & 60.76  & 15.86  & 69.44  & 28.51  & 67.17  & 25.70  & 65.79  & 23.36  \\
    Leaning & DRJL+MLP & 60.89  & 15.89  & 70.52  & 30.23  & 67.42  & 26.32  & 66.28  & 24.15  \\
          & ACL+LR & 60.76  & 15.80  & 70.50  & 30.05  & 65.78  & 23.49  & 65.68  & 23.11  \\
          & ACL+MLP & 60.71  & 16.68  & 70.47  & 29.95  & 66.07  & 23.90  & 65.75  & 23.51  \\
          & SRDO+LR & 60.85  & 15.90  & 70.06  & 29.50  & 63.22  & 20.12  & 64.71  & 21.84  \\
          & SRDO+MLP & 60.79  & 15.56  & 69.56  & 29.11  & 62.07  & 18.80  & 64.14  & 21.16  \\
    \midrule
          & Cross-Stitch-M & 61.10  & 16.40  & 70.59  & 30.33  & 66.86  & 25.18  & 66.18  & 23.97  \\
    Multi-Task & MMOE-M & 61.07  & 16.52  & 70.55  & 30.29  & 64.81  & 21.87  & 65.48  & 22.89  \\
    Learning & PLE-M & 61.16  & 16.61  & 70.62  & 30.45  & 63.52  & 20.94  & 65.10  & 22.67  \\
          & RMT-Net & 61.28  & 16.84  & 71.01  & 30.84  & 68.86  & 28.11  & 67.05  & 25.26  \\
          & RMT-Net++ & $\ \ $\textbf{61.60}$^*$ & $\ \ $\textbf{17.43}$^*$ & $\ \ $\textbf{71.96}$^*$ & $\ \ $\textbf{32.28}$^*$ & $\ \ $\textbf{70.41}$^*$ & $\ \ $\textbf{30.50}$^*$ & $\ \ $\textbf{67.99}$^*$ & $\ \ $\textbf{26.74}$^*$ \\
    \bottomrule
    \end{tabular}%
  \label{tab:comparison3}%
\end{table*}%

\subsection{Performance Comparison} \label{sec:comparison}

In our experiments, the performance comparison is conducted from three perspectives.

Firstly, we need to conduct performance comparison on approved samples.
Tab. \ref{tab:comparison1} illustrates performance comparison on approval-only datasets.
Overall speaking, semi-supervised, counterfactual and multi-task approaches achieve relative improvements compared with baselines.
This means, these approaches are effective for credit scoring, even when training and testing samples share similar data distributions.
Moreover, it is clear that RMT-Net achieves the best performances.
On average, RMT-Net relatively improves baselines by $9.32\%$, and improves the second-best compared approach by $3.61\%$, evaluated by KS.

Secondly, we need to investigate whether our proposed approaches can stably achieve promising performances in different data distributions.
This requires us to conduct experiments on both approved and rejected samples.
This is a very important experiment, for we strongly need credit scoring models that can stably and accurately infer credits of loan applications in feature distributions of both approved and rejected samples in real-world applications.
Tab. \ref{tab:comparison2} illustrates performance comparison on approval-rejection datasets.
It is clear that baselines achieve poor performances, due to the distribution shift between training and testing samples.
Both semi-supervised and counterfactual approaches significantly outperform baselines, and counterfactual approaches slightly outperform semi-supervised approaches.
Meanwhile, multi-task learning approaches except RMT-Net perform poorly, and some of them are even worse than baselines.
This could be because some task weights are not well learned, since we have no observed default/non-default labels for rejected samples during model training.
In the feature distribution of rejected samples, there is no supervision for optimizing the task weights to control how much information is shared from the rejection/approval task to the default/non-default task.
Instead, RMT-Net clearly achieves such supervision for optimizing the task weights with the best performances on all datasets.
Comparing with Cross-Stitch, MMOE, and PLE, RMT-Net relatively improves KS by $36.97\%$, $70.32\%$, and $72.62\%$ on average respectively.
Moreover, on average, RMT-Net relatively improves baselines by $76.80\%$, and improves the second-best compared approach by $18.21\%$, evaluated by KS.
These improvements are much larger than those on approval-only datasets.
This is because rejected samples are evaluated during testing. With a better learning of the default/non-default task for rejected samples, RMT-Net provides bigger room for improvements.

Thirdly, we also need to conduct performance comparison under multiple policies.
Tab. \ref{tab:comparison3} shows performance comparison on multi-policy datasets.
We can clearly observe that RMT-Net++ can further improve the performances of RMT-Net.
On average, RMT-Net++ relatively improves RMT-Net by $5.83\%$, evaluated by KS.

These experimental results strongly verify the effectiveness of our proposed RMT-Net and RMT-Net++.

\subsection{Hyper-parameter Study} \label{sec:params}

Furthermore, we are going to investigate the impact of hyper-parameters in our proposed RMT-Net and RMT-Net++.
In Fig. \ref{fig:params}, we illustrate the performances of RMT-Net and RMT-Net++ on testing set with varying loss balancing parameter $\lambda$ and layer number $t$.

\begin{figure}
\centering
\includegraphics[width=0.48\textwidth]{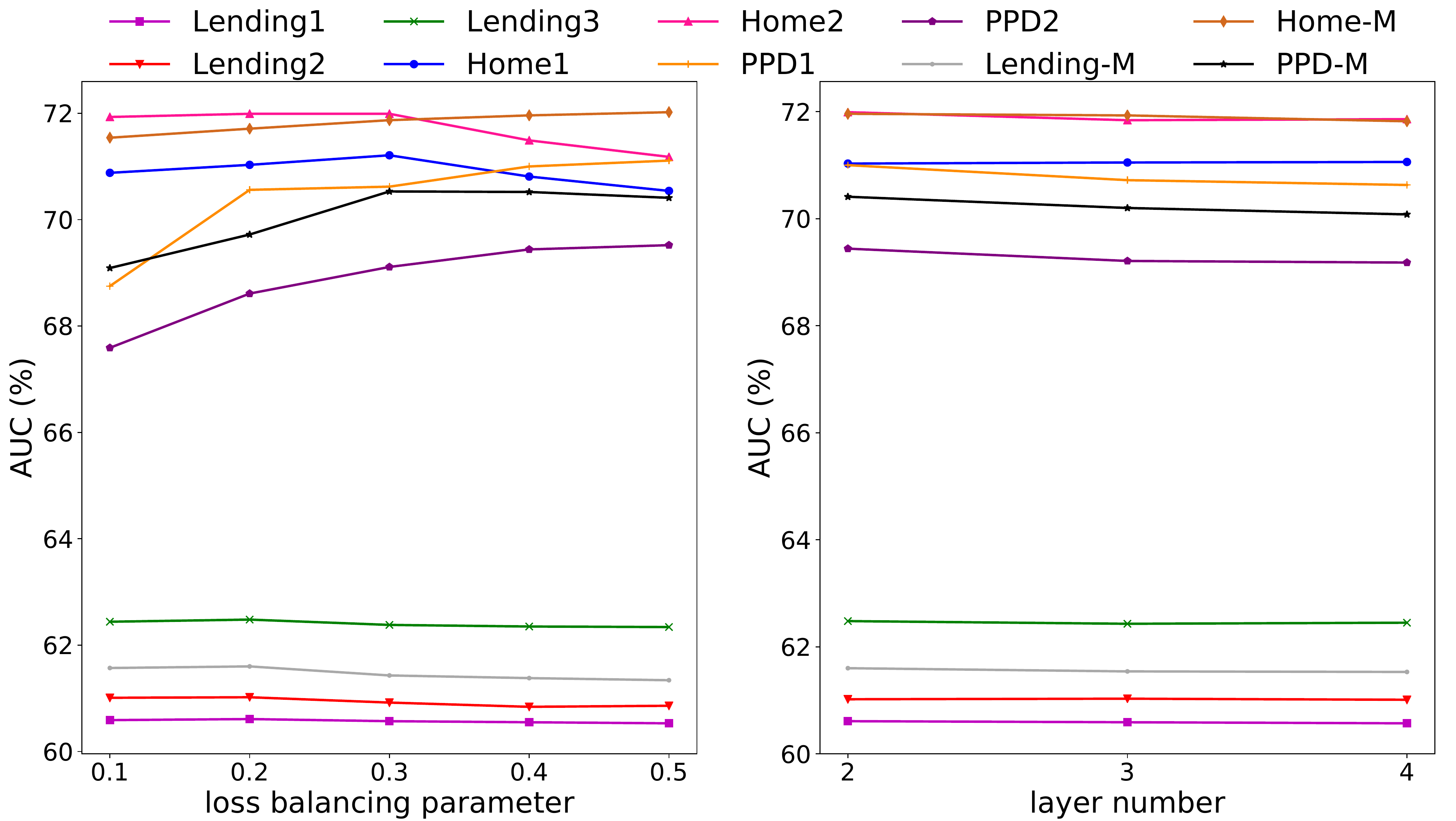}
\caption{Performances of RMT-Net and RMT-Net++ on testing set with varying hyper-parameters: (1) the left part shows the impact of the loss balancing parameter $\lambda$; (2) the right part shows the impact of the layer number $t$. To be noted, we report results of RMT-Net on single-policy dataset, and results of RMT-Net++ on multi-policy datasets.}
\label{fig:params}
\end{figure}

\begin{figure*}[t]
	\centering
	\subfigure[RMT-Net on approval-only datasets.]{
		\begin{minipage}[b]{0.45\textwidth}
			\includegraphics[width=1\textwidth]{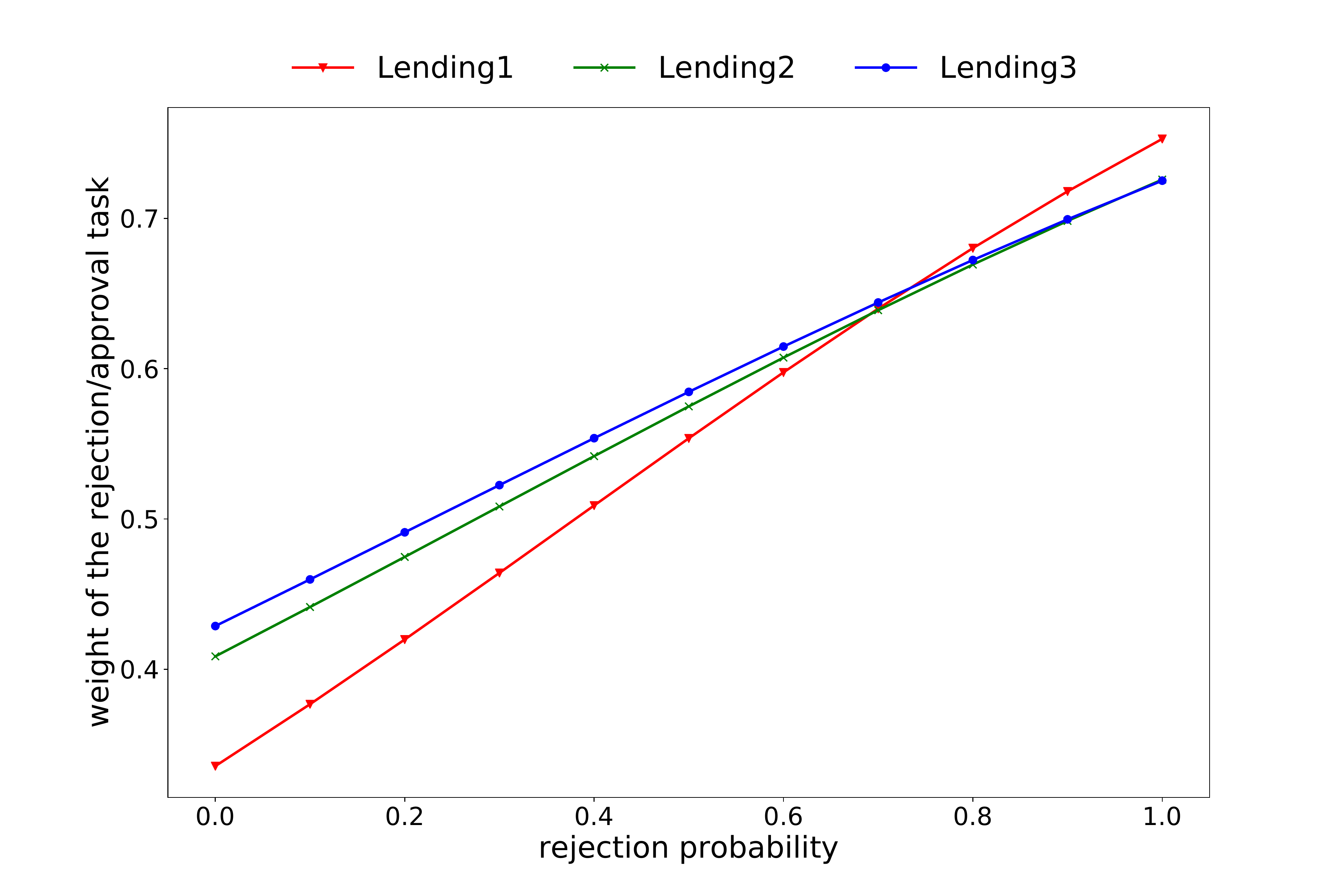}
		\end{minipage}
	}
	\subfigure[RMT-Net on approval-rejection datasets.]{
		\begin{minipage}[b]{0.45\textwidth}
			\includegraphics[width=1\textwidth]{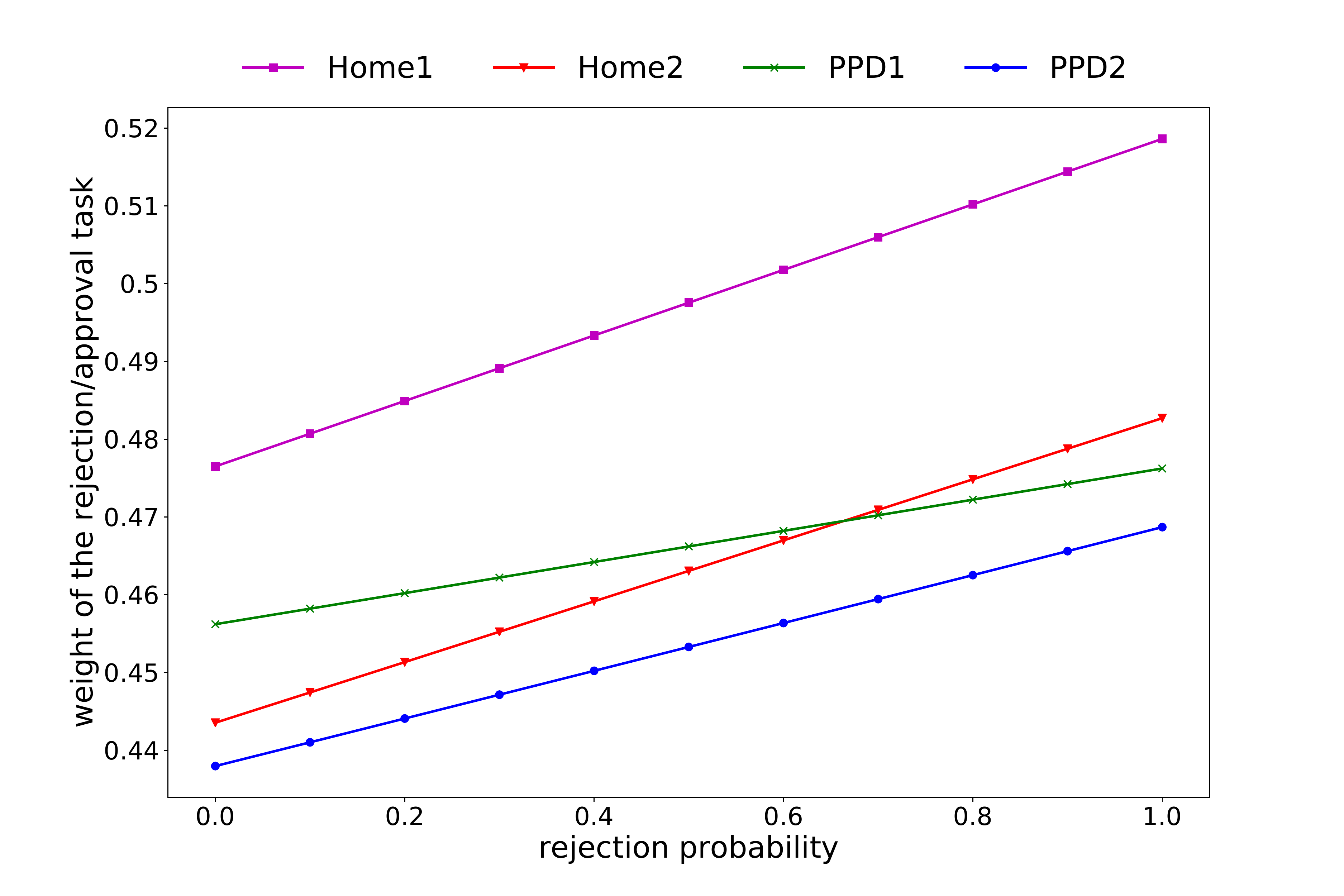}
		\end{minipage}
	}
	\subfigure[RMT-Net++ on the Lending-M dataset.]{
		\begin{minipage}[b]{0.45\textwidth}
			\includegraphics[width=1\textwidth]{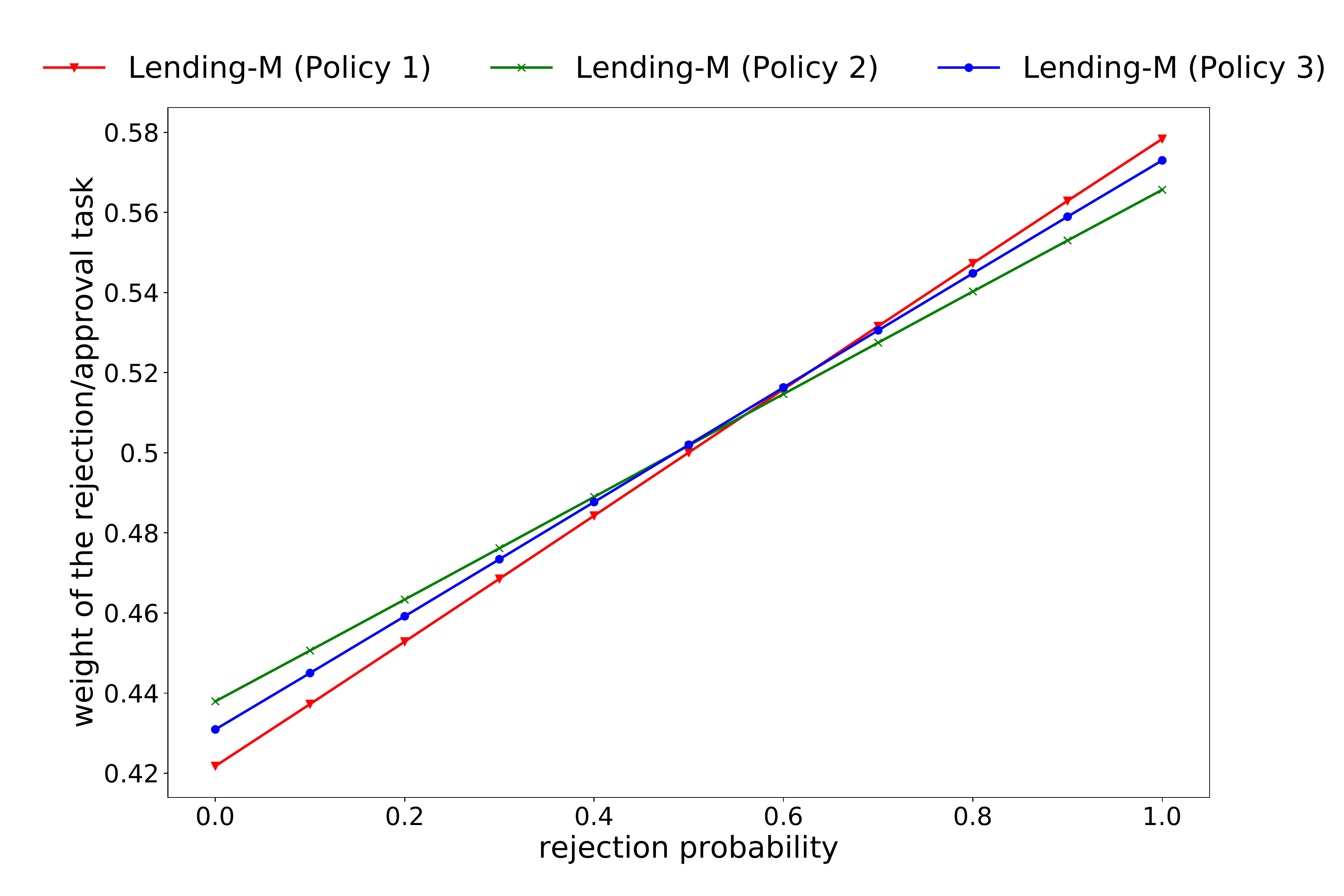}
		\end{minipage}
	}
	\subfigure[RMT-Net++ on the Home-M and PPD-M datasets.]{
		\begin{minipage}[b]{0.45\textwidth}
			\includegraphics[width=1\textwidth]{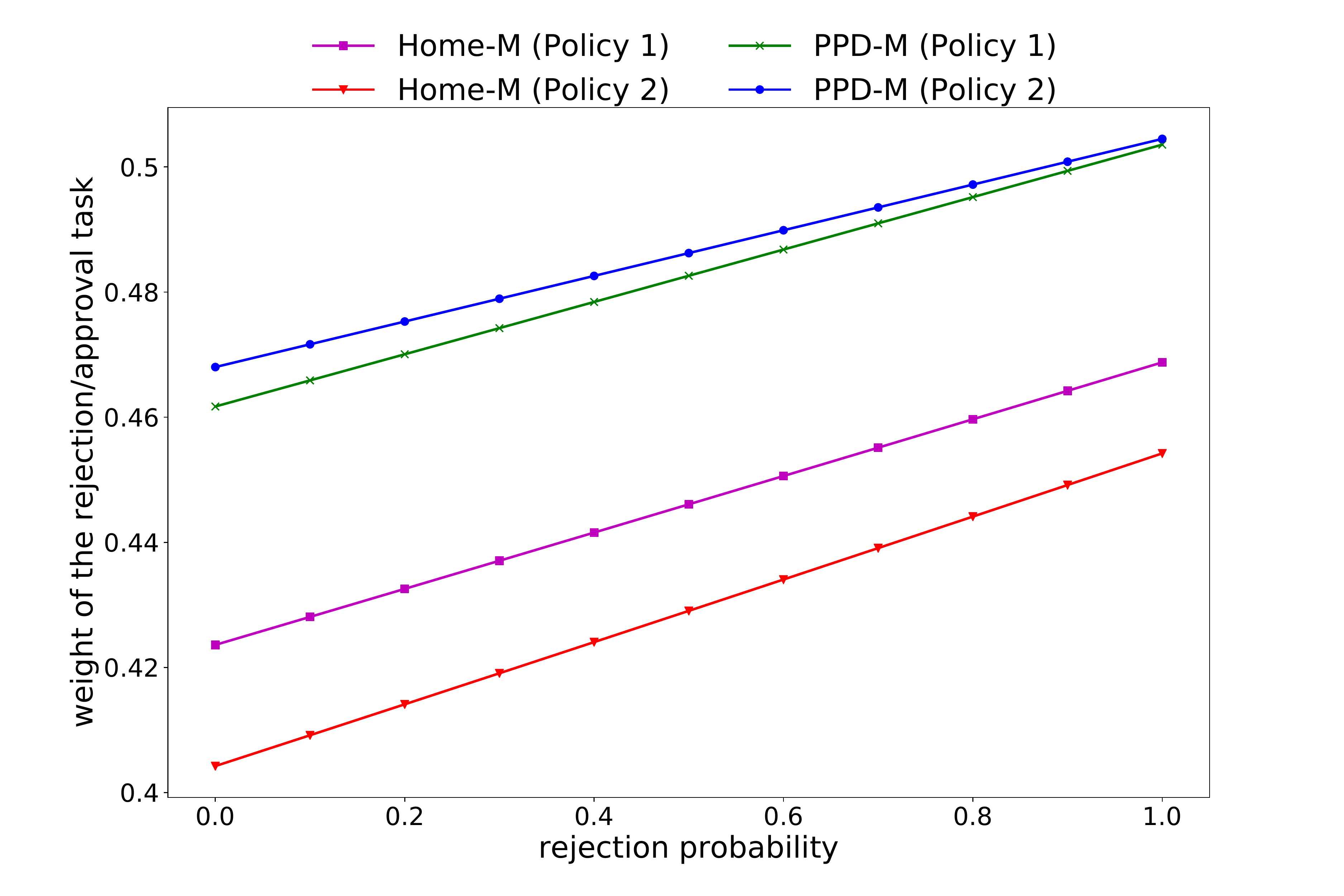}
		\end{minipage}
	}
	\caption{Visualization of gating networks in RMT-Net and RMT-Net++, in which we illustrate the relationship between the rejection probability and the weight of the rejection/approval task for the learning of the default/non-default task.}
	\label{fig:lines}
\end{figure*}

Firstly, the loss balancing parameter $\lambda$ somehow affects the performances of RMT-Net and RMT-Net++.
Thus, it is better for us to tune this hyper-parameter according to validation set for optimal performances.
In Sec. \ref{sec:comparison}, we report results on testing set via hyper-parameter tuning on validation set.
Moreover, the performances are not very sensitive to $\lambda$, and performance on each dataset stays stable in a range of loss balancing parameter $\lambda$.

Secondly, the layer number $t$ has very slight effects on the performances of RMT-Net and RMT-Net++.
Thus, we do not need to carefully tune this hyper-parameter.
To be noted, $t$ includes the final layer in MLP, and the minimum value of $t$ is $2$.
This means, $t$ layers indicate we have $t-1$ hidden layers in RMT-Net and RMT-Net++ for each task.
If we set $t=1$, there will be no hidden layers in RMT-Net and RMT-Net++, and our design reject-aware multi-task learning framework will be invalid.
Accordingly, for simplicity, we set $t=2$ in rest of our experiments, and report according results on testing set in Sec. \ref{sec:comparison}.

\subsection{Visualization}

In Fig. \ref{fig:lines}, we illustrate the relationship between the rejection probability and the weight of the rejection/approval task for the learning of the default/non-default task in RMT-Net and RMT-Net++.
This demonstrates the status of gating networks.
For all datasets, the larger the rejection probability, the larger the weight of the rejection/approval task.
This means our proposed approaches learns that, with a larger rejection probability, less reliable information can be learned in the default/non-default network, and more information should be shared from the rejection/approval network.
With such a gating network, we can alleviate the under-fitting problem of conventional multi-task approaches in the feature distribution of rejected samples.

\section{Conclusion}

In this paper, we focus on modeling biased credit scoring data, in which we have only ground-truth labels for approved samples and no observations for rejected samples.
Such bias affects the reliability of default prediction, and we aim to improve the prediction accuracy on both approved and rejected samples. 
We find that the default/non-default classification task and the rejection/approval classification task are highly correlated in credit scoring applications, according to both real-world data study and theoretical analysis.
We for the first time propose to model biased credit scoring data using an MTL framework, and propose a novel RMT-Net approach, which learns the task weights that control the information sharing from the rejection/approval task to the default/non-default task by a gating network based on rejection probabilities.
According to empirical experiments on $10$ datasets under different settings, RMT-Net improves the poor performances of existing MTL approaches, and significantly outperforms several state-of-the-art approaches from different perspectives.
Furthermore, we extend RMT-Net to RMT-Net++ for modeling scenarios with multiple rejection/approval strategies.
According to an extra experiment, RMT-Net++ with multiple strategies can further improve the performances of RMT-Net in a more complex multi-policy scenario.

\section*{Acknowledgments}

The authors would like to thank the anonymous reviewers for their valuable comments and suggestions allowing them to improve the quality of this paper. 
This work is jointly sponsored by National Natural Science Foundation of China (U19B2038, 62141608) and CCF-AFSG Research Fund (20210001).

\ifCLASSOPTIONcaptionsoff
  \newpage
\fi

\bibliographystyle{IEEEtran}
\bibliography{IEEEtran.bib}

\end{document}